\documentclass{article} % For LaTeX2e
\usepackage{iclr2026_conference}
\usepackage{times}

\usepackage{amsmath,amsfonts,bm}

\def\eqref#1{equation~\ref{#1}}
\def\1{\bm{1}}

\DeclareMathAlphabet{\mathsfit}{\encodingdefault}{\sfdefault}{m}{sl}
\SetMathAlphabet{\mathsfit}{bold}{\encodingdefault}{\sfdefault}{bx}{n}

\newcommand{\E}{\mathbb{E}}

\newcommand{\R}{\mathbb{R}}

\newcommand{\KL}{D_{\mathrm{KL}}}

\usepackage{hyperref}
\usepackage{url}
\usepackage{enumitem}
\usepackage{amsmath,amssymb,amsthm,mathtools}
\usepackage{bm}
\usepackage{hyperref}
\usepackage{cleveref}
\usepackage{algorithm}
\usepackage{algorithmic}
\usepackage{booktabs}
\usepackage{xcolor}
\usepackage{subfig}
\usepackage{float}
\usepackage{tikz}
\usetikzlibrary{arrows.meta,positioning,shapes.geometric,fit,calc}

\newtheorem{theorem}{Result}

\newtheorem{corollary}{Corollary}
\newtheorem{definition}{Definition}
\newtheorem{assumption}{Assumption}
\newtheoremstyle{remarkstyle}  % name of the style
  {3pt}                         % space above
  {3pt}                         % space below
  {\normalfont}                 % body font (upright)
  {}                            % indent amount
  {\bfseries}                   % heading font (bold)
  {.}                           % punctuation after heading
  { }                           % space after heading
  {}                            % heading specification

\theoremstyle{remarkstyle}

\newcommand{\Prob}{\mathbb{P}}
\newcommand{\tr}{\mathrm{tr}}

\newcommand{\Mc}{\mathcal{M}_c}
\newcommand{\Mh}{\mathcal{M}_h}
\newcommand{\Ac}{\mathbf{A}_c}
\newcommand{\Ah}{\mathbf{A}_h}
\newcommand{\bDelta}{\boldsymbol{\Delta}}
\newcommand{\bSigma}{\boldsymbol{\Sigma}}
\newcommand{\bxi}{\boldsymbol{\xi}}
\newcommand{\bz}{\mathbf{z}}
\newcommand{\by}{\mathbf{y}}
\newcommand{\bx}{\mathbf{x}}
\newcommand{\bP}{\mathbf{P}}
\newcommand{\bG}{\mathbf{G}}
\newcommand{\bI}{\mathbf{I}}

\newcommand{\norm}[1]{\left\|#1\right\|}
\newcommand{\mnorm}[2]{\left\|#1\right\|_{#2}}

\DeclareMathOperator{\TV}{TV}

\title{Guarantees on Dynamical System\\ Distinguishability for LLM Token Generation}

\author{Mohamed Akrout\thanks{Equal contribution}~~, Dan Wilson$^*$ \\
Department of Electrical Engineering and Computer Science\\
University of Tennessee, Knoxville, TN, USA\\
\texttt{\{mohamed.akrout,dan.wilson\}@tennessee.edu}
}

\iclrfinalcopy % Uncomment for camera-ready version, but NOT for submission.
\begin{document}

\maketitle

    \begin{abstract}
    Recent work has shown that classifying large language models (LLMs)' responses can be distinguished by modeling token embeddings as trajectories of a black-box dynamical system (DS) and comparing prediction residuals of two DSs. Despite the empirical success of this dynamical approach, a theoretical understanding of \emph{why} it works, \emph{how well} it scales as a function of the token sequence, and \emph{when} it transfers across embedding models remains lacking. We address these questions by formalizing the classification task as a binary hypothesis test between two stochastic linear DSs. We show that the total variation distance between the stationary marginal distributions of the two DSs can be arbitrarily small even when the dynamics differ substantially, which provides a fundamental accuracy floor for any classifier that ignores token dynamics. We then show that the misclassification probability of DS-based classification decays exponentially in the sequence length~$L$, with the decay governed by a \emph{dynamical discriminability} quantity $\delta^2$ that captures the spectral distance between the two DSs. We also characterize cross-embedding generalization by introducing an \emph{approximate intertwining} condition between embedding models and establishing a lower bound on the transferable discriminability in terms of the intertwining map's smallest singular value. Together, these results explain the empirical performance of DS-based classification and motivate further investigation into using DS theory to analyze AI systems, in contrast to the more common approach of using AI to model dynamical systems.\vspace{-0.2cm}
    \end{abstract}
    
    \section{Introduction}
    The problem of distinguishing between two dynamical systems from their output trajectories arises in numerous applications, including change-point detection in control systems~\citep{xin2025online}, fault diagnosis in engineering~\citep{sheikhi2025fault}, and, more recently, hallucination detection in Large Language Models (LLMs) \citep{wilson2026low}. In the LLM context, recent work has demonstrated that factual and hallucinated text can be distinguished by modeling the evolution of token embeddings as trajectories of a black-box dynamical system and fitting separate Koopman operator approximations for each regime \citep{wilson2026low}.
    
    \begin{figure}[h!]
        \vspace{-0.2cm}
        \subfloat[Multiple LLM responses]{{
        \includegraphics[width=0.54\linewidth]{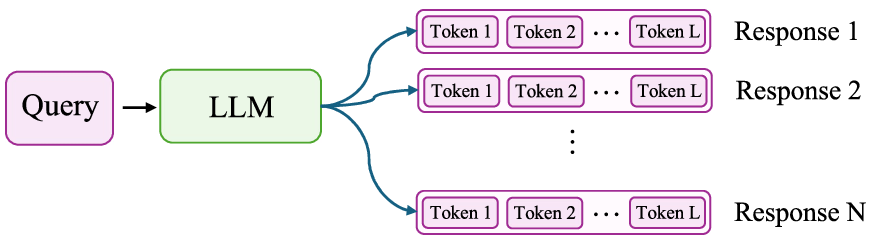}
        \label{fig:multi-response}}}
        \vspace{0.1cm}
        \subfloat[token dynamics of one LLM response]{{
        \includegraphics[width=0.44\linewidth]{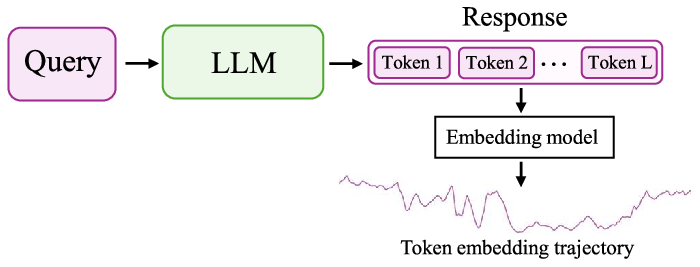}
        \label{fig:response-dynamics}}}%    
        \caption{The two approaches to analyze the properties of LLM responses: (a) multiple responses where tokens are assumed i.i.d. within responses, and (b) one single response where the token dynamics is considered.}
        \label{fig:multiple-vs-single-response}
    \end{figure}
    
    Existing approaches for classifying LLM responses (e.g., distinguishing safe from unsafe, or correct from hallucinated) require multiple responses to the same query as shown in Fig. \ref{fig:multi-response} in order to make a reliable decision \cite{lewis2020retrieval,manakul2023selfcheckgpt,zhang2023sac3,dhuliawala2024chain,meng2022locating}. This is because they treat the tokens within each generated sentence as independent and identically distributed (i.i.d.), and therefore ignore the embedding evolution of the tokens across the response. As a result, these methods must aggregate statistics over several independent generations to compensate for the loss of token dynamics, which increases both latency and API costs. By accounting for the rich token dynamics of a single LLM response in an embedding space as shown in Fig. \ref{fig:response-dynamics}.

    Despite empirical demonstrations that token embedding dynamics enable LLM response classification \citep{wilson2026low}, the theory behind this remains unexplored. This raises fundamental questions that we address in this work:
    \begin{enumerate}[leftmargin=*]
        \item[$i)$] \textit{Why do dynamics matter?} If the marginal distributions of token embeddings differ between two response categories, why is a dynamical model necessary?
        \item[$ii)$] \textit{How well does the dynamical system method work?} What is the classification error as a function of sequence length of the response, and what dynamical system quantities govern the classification performance?
        \item[$iii)$] \textit{Are embedding dynamics generalizable?} To what extent do token embeddings transfer across embedding models, and what conditions govern their cross-model generalization? From a DS perspective, embedding transferability is a striking and unexpected finding. This is because, at a fundimental level, data driven methods for Koopman operator approximation depend on the specific observables used in the model fitting.  Even minor permutations to the ordering of the observable coordinates can ruin the predictive ability; switching to entirely different embedding model represents an even more drastic transformation. The fact that discriminative information survives such a change requires a theoretical explanation and has practical implications for reducing the cost of fitting models across different embedding backends.
    \end{enumerate}
    
    In this paper, we provide rigorous answers to all three questions. Our contributions are:
    \begin{itemize}[leftmargin=*]
        \item \textbf{Result~\ref{thm:static_limit} (Limitation of marginal density analysis):} We show that two stable DSs can have completely different dynamics yet produce stationary distributions that are nearly indistinguishable. This means that any classifier relying only on marginal statistics has a fundamental accuracy floor and cannot do better than random guessing when the distributions are too similar. This explains why static features like SVD modes often overlap for different embedding trajectory classes.
        
        \item \textbf{Result~\ref{thm:dynamical_sep} (Dynamical separability):} We prove that a classifier based on prediction residuals can achieve misclassification error that decays exponentially with sequence length, with the rate decay determined by how dynamically distinct the two DSs are and the noise level. This breaks the accuracy ceiling of classifers with marginal density analysis obtained in Result \ref{thm:static_limit}. We also show that reliable classification with error probabilty $\epsilon$ requires a number of tokens $L^*=\mathcal{O}\big(\log\frac{1}{\epsilon}\big)$, which explains why richer embedding models already achieve high accuracy with shorter sequences.
        \item \textbf{Result~\ref{prop:cross} (Cross-embedding transfer):} We prove that dynamical systems fitted on one embedding can be transferred to another, provided the two embedding spaces share similar geometric structure described by an affine transformation. The lower-bound on the ability to discriminate between classes is dependent on the singular value of the transformation matrix and the norm of the operator difference of the two DSs. This explains why cross-embedding classification performs above chance even without retraining.
    \end{itemize} 
    
    \textit{Notation}: We use bold lowercase for vectors ($\by, \bz$), bold uppercase for matrices ($\mathbf{A}, \bP$), and calligraphic letters for manifolds and sets ($\Mc, \Mh$). The operator $\norm{\cdot}$ denotes the Euclidean 2-norm for vectors and the spectral norm for matrices; $\norm{\cdot}_F$ denotes the Frobenius norm. For a positive definite matrix $\mathbf{M}$, we write $\mnorm{\bx}{\mathbf{M}} = \sqrt{\bx^\top\mathbf{M}\,\bx}$ for the Mahalanobis norm of $\bx$. We write $\tr(\cdot)$ for the trace, $\lambda_{\max}(\cdot)$ and $\lambda_{\min}(\cdot)$ for the largest and smallest eigenvalues of a symmetric matrix, $\sigma_{\min}(\cdot)$ for the smallest singular value, and $\rho(\cdot)$ for the spectral radius. We use $\mathcal{O}(\cdot)$ to denote asymptotic upper bounds and $\Theta(\cdot)$ for asymptotically tight bounds in both directions, both up to constant factors.
    
    \section{Prior work}
    \textbf{Distinguishability of dynamical systems.}  The problem of distinguishing between dynamical systems from observed data has a rich history in DS theory and signal processing. Classical system identification theory \citep{ljung1999system} addresses the problem of estimating system parameters from input-output data, while the related problem of \emph{hypothesis testing} between candidate models has been studied in the context of fault detection and isolation (\cite{sheikhi2025data}). In the linear setting, the question of whether two systems can be distinguished from output trajectories is closely tied to the notion of \emph{observability} (\cite{kailath1980linear}): two systems with distinct dynamics but identical output distributions are observationally indistinguishable. Recent work has provided finite-sample guarantees for learning linear dynamical systems from a single trajectory (\cite{simchowitz2018learning,sarkar2019near}), establishing that the sample complexity depends on the spectral radius and the signal-to-noise ratio of the system. Online change-point detection between linear dynamical regimes has been studied in (\cite{xin2025online}), who provide finite-sample guarantees for detecting switches between system parameters. Our work differs from these in that we characterize the error as a function of the dynamical discriminability between two dynamical systems.
    
    \textbf{Dynamical systems for neural networks.} The perspective of treating neural network inference as a dynamical system has gained traction in recent years. Neural Ordinary Differential Equations (\cite{chen2018neural}) established a formal connection between residual networks and continuous-time dynamical systems, while Koopman operator theory has been applied to analyze and predict the evolution of nonlinear systems from data (\cite{budisic2012applied,mezic2013analysis,will15}). In the LLM context,, token embedding sequences of LLM responses were treated as trajectories of a black-box dynamical system and fitted to Koopman-based linear models for factual and hallucinated regimes \cite{wilson2026low}, enabling single-response classification. However, the theoretical foundations of this approach remain unclear. The present work fills this gap.
    
    \section{Problem setup and outline}
    
    \subsection{Dynamical systems for classification}\label{sec:DS-for-classification}
    We consider a discrete-time system generating a sequence of observables $\by_1, \by_2, \ldots, \by_L \in \R^d$ through an embedding map applied to the output of a generative model (e.g., an LLM). The observables are obtained by projecting the raw embeddings onto the top $d$ singular value decomposition (SVD) modes:
    \begin{equation}
        \by_k = \boldsymbol{\Phi}^\top H(\bx_k) \in \R^d,
    \end{equation}
    where $\bx_k$ is the internal state of the generative model, $H$ is the composition of the token selection and embedding maps, and $\boldsymbol{\Phi} \in \R^{M \times d}$ is the SVD projection matrix with $M$ being the raw embedding dimension. Following the Extended Dynamic Mode Decomposition (EDMD) framework~\citep{will15}, we lift the observables to a higher-dimensional space:
    \begin{equation}
        \bz_k = \begin{bmatrix} \by_k \\ f_{\text{lift}}(\by_k) \end{bmatrix} \in \R^{d+\gamma},
    \end{equation}
    where $f_{\text{lift}}: \R^d \to \R^\gamma$ is a nonlinear lifting function (e.g., polynomial combinations of the observable components).  A set of snapshot pairs $s_k = (\mathbf{z}_k,\mathbf{z}_{k+1})$ is collected and arranged into matrices $\bm{X} = \begin{bmatrix} \mathbf{z}_1 & \dots & \mathbf{z}_{q}  \end{bmatrix}$ and $\bm{X}^+ = \begin{bmatrix} \mathbf{z}_2 & \dots & \mathbf{z}_{q+1}  \end{bmatrix}$ where $q$ is the number of snapshot pairs used. An approximation of the Koopman operator $\mathbf{z}_{k+1} = \mathbf{A} \mathbf{z}_k$ can be obtained according to $\mathbf{A} = \mathbf{X}^+  \mathbf{X}^\dagger$, where $^\dagger$ denotes the pseudoinverse.  A prediction for the evolution of the token embeddings can be obtained according to $\tilde{\mathbf{y}}_{k+1} = \begin{bmatrix} \bm{I}  & \bm{0} \end{bmatrix}  \bm{A} \bm{z}_k$ where $\bm{I} \in \mathbb{R}^{M \times M}$ denotes the identity matrix, $\bm{0} \in \mathbb{R}^{M \times \gamma}$ is a matrix of zeros.  
    
    In this work, we obtain distinct approximations of the Koopman operator, denoted $\mathbf{A}_1$ and $\mathbf{A}_2$, from data belonging to two different classes. In the LLM context, and without loss of generality, these correspond to the approximated operators $\mathbf{A}_h$ and $\mathbf{A}_c$ which we estimate from hallucinated and correct responses, respectively. Given an arbitrary LLM response, with sequences of token embeddings $\bz_1, \ldots, \bz_L$, the prediction errors under each model are:
    \begin{align}
        \epsilon_{c,k} &\equiv \norm{\by_{k+1} - \bP \Ac \bz_k}, \\
        \epsilon_{h,k} &\equiv \norm{\by_{k+1} - \bP \Ah \bz_k},
    \end{align}
    where $\bP = [\bI_d \;\; \mathbf{0}_{d \times \gamma}]$ extracts the observable components. The cumulative squared errors are:
    \begin{equation}
        S_c = \sum_{k=1}^{L-1} \epsilon_{c,k}^2, \qquad S_h = \sum_{k=1}^{L-1} \epsilon_{h,k}^2,
    \end{equation}
    and the differential residual score is:
    \begin{equation} \label{eq:resid}
        \Delta\mathcal{E} = \sqrt{S_h} - \sqrt{S_c}.
    \end{equation}
    The decision rule classifies the trajectory as hallucinated if $\Delta\mathcal{E} < \eta$ for a classification threshold $\eta$.

    \subsection{Classification as a hypothesis testing problem}
    
    We model the classification task as a binary hypothesis test. Under each hypothesis, the lifted observables evolve according to a stochastic linear dynamical system:
    \begin{align}
        \mathcal{H}_c: \quad \bz_{k+1} &= \Ac \bz_k + \bxi_k^{(c)}, \qquad \bxi_k^{(c)} \sim \mathcal{N}(\mathbf{m}_c, \bSigma_c), \label{eq:Hc} \\
        \mathcal{H}_h: \quad \bz_{k+1} &= \Ah \bz_k + \bxi_k^{(h)}, \qquad \bxi_k^{(h)} \sim \mathcal{N}(\mathbf{m}_h, \bSigma_h), \label{eq:Hh}
    \end{align}
    where $\Ac, \Ah \in \R^{(d+\gamma) \times (d+\gamma)}$ are the finite-dimensional approximations of the Koopman operator for the correct (factual) and hallucinated regimes, respectively. The noise terms $\bxi_k^{(c)}$ and $\bxi_k^{(h)}$ capture both the approximation error of the finite-dimensional Koopman truncation and any genuine stochasticity in the generation process. We allow each regime to have its own noise mean ($\mathbf{m}_c \in \R^{d+\gamma}$ or $\mathbf{m}_h \in \R^{d+\gamma}$) and its own noise covariance ($\bSigma_c$ or $\bSigma_h$), reflecting the fact that the fidelity of the linear approximation may differ between factual and hallucinated dynamics.
    
    \begin{assumption}[Stability]
    \label{ass:stability}
    Both systems are stable: $\rho(\Ac) < 1$ and $\rho(\Ah) < 1$.
    \end{assumption}
    
    \begin{assumption}[Gaussian noise]
    \label{ass:gaussian}
    The noise sequences $\{\bxi_k^{(c)}\}$ and $\{\bxi_k^{(h)}\}$ are i.i.d.\ Gaussian with means $\mathbf{m}_c, \mathbf{m}_h \in \R^{d+\gamma}$ and positive definite covariance matrices $\bSigma_c, \bSigma_h \in \R^{(d+\gamma)\times(d+\gamma)}$, respectively.
    \end{assumption}
    
    \noindent Under Assumptions~\ref{ass:stability} and~\ref{ass:gaussian}, each system possesses a unique stationary distribution. The \emph{stationary means} are:
    \begin{equation}
        \boldsymbol{\mu}_c = (\bI - \Ac)^{-1}\mathbf{m}_c, \qquad \boldsymbol{\mu}_h = (\bI - \Ah)^{-1}\mathbf{m}_h,
        \label{eq:stationary_means}
    \end{equation}
    where the invertibility of $\bI - \Ac$ and $\bI - \Ah$ is guaranteed by $\rho(\Ac) < 1$ and $\rho(\Ah) < 1$. The \emph{stationary covariances} $\bG_c$ and $\bG_h$ are the unique positive definite solutions to the respective discrete Lyapunov equations:
    \begin{equation}
        \bG_c = \Ac \bG_c \Ac^\top + \bSigma_c, ~\,\textrm{and}~\, \bG_h = \Ah \bG_h \Ah^\top + \bSigma_h.
        \label{eq:lyapunov}
    \end{equation}
    The covariances in (\ref{eq:lyapunov}) are equivalently given by the convergent series $\bG_c = \sum_{j=0}^{\infty} \Ac^j \bSigma_c (\Ac^j)^\top$ and $\bG_h = \sum_{j=0}^{\infty} \Ah^j \bSigma_h (\Ah^j)^\top$. Note that the existence, convergence, and uniqueness of these solutions is a classical result in linear systems theory (\citet{hamilton2020time,kailath1980linear}).

    \begin{figure}[h!]
        \centering 
        \vspace{-0.2cm}
        \includegraphics[scale=0.45]{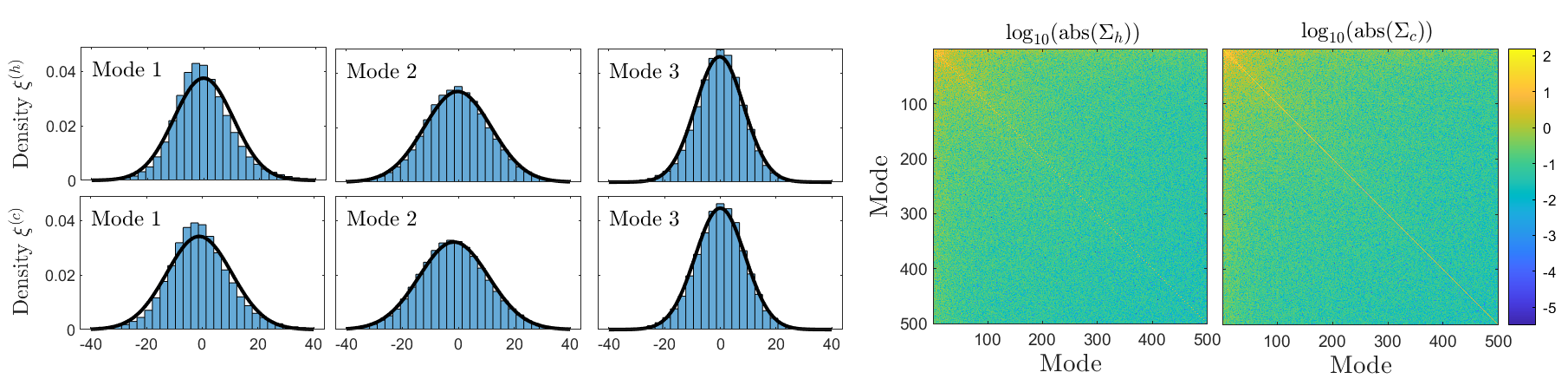}    
        \caption{Extended DMD is applied to the HaluEval dataset~\cite{li2023halueval} dataset.  80,000 samples of $\bxi_k^{(h)}$ and $\bxi_k^{(c)}$ are taken; the distributions are well-approximated by Gaussian curves (left) with mean and variance chosen appropriately.   An estimate of the covariances is also provided (right).}
        \label{fig:distributions}
        %\vspace{-0.3cm}
    \end{figure}

    When applied to the HaluEval~\cite{li2023halueval} dataset, extended DMD yields predictive models that are well-approximated by \eqref{eq:Hc} and \eqref{eq:Hh}.  After fitting matrices $\mathbf{A}_h$ and $\mathbf{A}_c$ with 20\% of the dataset, 80,000  samples of $\bxi_k^{(h)}$ and $\bxi_k^{(c)}$ are taken over hallucinated responses, with distributions of the responses shown for the first 3 SVD modes in \ref{fig:distributions}.  Gaussian curves with identical mean and variance are superimposed highlighting that the distribution of errors is well approximated by a normal distribution.  The noise covariance is also computed and included in Fig. \ref{fig:distributions}. 
    
    \subsection{Outline}
    Fig. ~\ref{fig:roadmap} summarizes the logical dependencies between the key DS quantities already defined and our three main results.
    
    \begin{figure}[h!]
    \centering
    \vspace{-0.2cm}
    \begin{tikzpicture}[
        box/.style={draw, rounded corners, minimum width=2.2cm, minimum height=0.7cm, align=center, font=\small},
        result/.style={draw, rounded corners, fill=blue!8, minimum width=2.8cm, minimum height=0.8cm, align=center, font=\small\bfseries},
        arr/.style={-{Stealth[length=2.5mm]}, thick},
        every node/.style={inner sep=3pt},
        scale=1, transform shape
    ]
    % System parameters (centered)
    \node[box] (params) at (0,0) {Dynamical System parameters\\$\Ac,\Ah,\bSigma_c,\bSigma_h,\mathbf{m}_c,\mathbf{m}_h$};
    
    % Stationary distributions (left)
    \node[box] (stat) at (-4.8,-1.5) {Stationary distributions\\$p_c = \mathcal{N}(\bP\boldsymbol{\mu}_c, \boldsymbol{\Gamma}_c)$\\$p_h = \mathcal{N}(\bP\boldsymbol{\mu}_h, \boldsymbol{\Gamma}_h)$};
    
    % Dynamical discriminability (center)
    \node[box] (delta) at (0,-1.5) {Dynamical discriminability\\$\delta^2 = \E[\|\bP\bDelta\bz_k\|^2]$};
    
    % Intertwining (right)
    \node[box] (intw) at (4.8,-1.5) {Intertwining map\\$G_2^{(\boldsymbol{\theta}_2)} \approx \mathbf{T}\, G_2^{(\boldsymbol{\theta}_1)} + \boldsymbol{r}$};
    
    % Results
    \node[result] (R1) at (-4.8,-3.2) {Result~\ref{thm:static_limit}\\$d_{\TV}$ can be $\approx 0$\\$\Rightarrow P_{\mathrm{error}}^{\mathrm{static}} \approx 1/2$};
    \node[result] (R2) at (0,-3.2) {Result~\ref{thm:dynamical_sep}\\$P_{\mathrm{error}} \leq e^{-cL\delta^2/\sigma^2}$\\$\Rightarrow$ accuracy $\uparrow$ with $L$};
    \node[result] (R3) at (4.8,-3.2) {Result~\ref{prop:cross}\\$\delta^2_{\boldsymbol{\theta}_1\!\to\!\boldsymbol{\theta}_2} \geq \sigma_{\min}^2 \delta^2 - \mathcal{O}(\chi^2)$\\$\Rightarrow$ transfer possible};
    
    % Arrows from params to middle layer
    \draw[arr] (params) -- (stat);
    \draw[arr] (params) -- (delta);
    \draw[arr] (params) -- (intw);
    
    % Arrows from middle layer to results
    \draw[arr] (stat) -- (R1);
    \draw[arr] (delta) -- (R2);
    \draw[arr] (intw) -- (R3);
    
    % Additional arrow from delta to R3
    \draw[arr] (delta) -- (R3);
    
    % Dashed arrows with more space
    \draw[arr, dashed, gray] (R1.east) -- node[above, font=\itshape, text=gray, sloped] {motivates} (R2.west);
    \draw[arr, dashed, gray] (R2.east) -- node[above, font=\itshape, text=gray, sloped] {extends} (R3.west);
    
    \end{tikzpicture}
    \caption{Logical dependencies between the key DS variables and the obtained results.}
    \label{fig:roadmap}
    \end{figure}
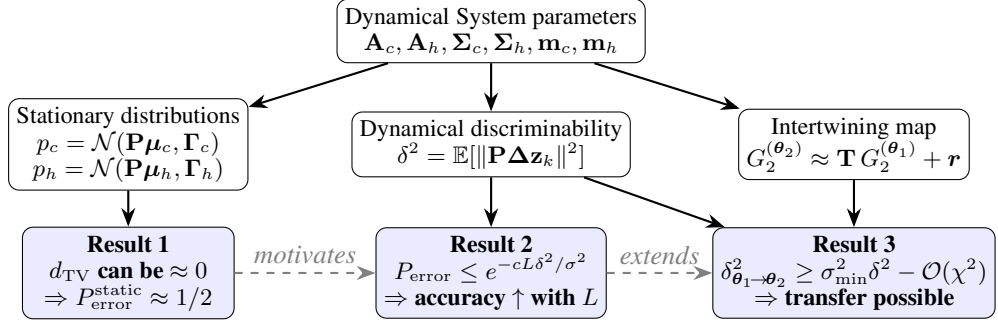
    
    \vspace{-0.3cm}
    
    \section{The Dynamical Discriminability}
    Given a high-dimensional embedding token sequence $\{\bm{z}_k\}_{\ell=1}^L$ of length $L$, we seek to determine whether it originates from dynamical system ($\mathcal{H}_c$ or $\mathcal{H}_h$) and characterize how the distinguishability between the two depends on $L$ and the parameters of the dynamical systems in (\ref{eq:Hc}) and (\ref{eq:Hh}). Toward this goal, we define the \emph{operator difference}:
    \begin{equation}
        \bDelta \equiv \Ac - \Ah.
    \end{equation}
    
    \begin{definition}[Dynamical discriminability]
    \label{def:discriminability}
    The \emph{dynamical discriminability} of the embedding under $\mathcal{H}_c$ is:
    \begin{equation}
        \delta^2 \equiv \E\!\left[\norm{\bP \bDelta \bz_k}^2\right] = \tr\!\left(\bP \bDelta \bG_c \bDelta^\top \bP^\top\right) + \norm{\bP \bDelta \boldsymbol{\mu}_c}^2,
        \label{eq:discriminability}
    \end{equation}
    where $\bG_c$ is the stationary covariance and $\boldsymbol{\mu}_c$ is the stationary mean under $\mathcal{H}_c$. The \emph{mean-bias interaction} is:
    \begin{equation}
        \beta_c \equiv 2\,\mathbf{m}_c^\top \bP^\top \bP \bDelta \boldsymbol{\mu}_c.
        \label{eq:beta_c}
    \end{equation}
    \end{definition}
    
    The discriminability $\delta^2 \geq 0$ measures the expected squared model-mismatch signal per time step and represents the systematic error that arises when the wrong model $\Ah$ is used to predict a trajectory generated by $\Ac$. The mean-bias interaction $\beta_c$ (which can be positive, negative, or zero) captures the correlation between the noise mean and the model-mismatch signal evaluated at the stationary operating point. When the noise has zero mean ($\mathbf{m}_c = \mathbf{0}$), we have $\boldsymbol{\mu}_c = \mathbf{0}$ and $\beta_c = 0$, and the discriminability reduces to $\delta^2 = \tr(\bP\bDelta\bG_c\bDelta^\top\bP^\top)$.
    
    To build intuition, consider what happens when we use the \emph{wrong} model to predict the next observation. Suppose the true dynamics follow $\mathcal{H}_c$, so that $\bz_{k+1} = \Ac \bz_k + \bxi_k^{(c)}$. If we instead predict using $\Ah$, the prediction for the observable component is $\widehat{\by}_{k+1} = \bP \Ah \bz_k$, while the true observable evolves as $\by_{k+1} = \bP \Ac \bz_k + \bP \bxi_k^{(c)}$. The prediction error therefore decomposes as:
    \begin{equation}
        \by_{k+1} - \widehat{\by}_{k+1} = \bP(\Ac - \Ah)\bz_k + \bP\bxi_k^{(c)} = \underbrace{\bP \bDelta \bz_k}_{\text{model-mismatch signal}} + \underbrace{\bP \bxi_k^{(c)}}_{\text{noise}}.
        \label{eq:error_decomposition}
    \end{equation}
    The model-mismatch signal $\bP\bDelta\bz_k$ has two components: a random component driven by the stationary covariance $\bG_c$, and a deterministic offset $\bP\bDelta\boldsymbol{\mu}_c$ arising from the stationary mean. The expected squared magnitude of this signal is:
    \begin{equation}
        \begin{aligned}
        \delta^2 = \E\!\left[\norm{\bP \bDelta \bz_k}^2\right] \stackrel{(a)}{=} \E\!\left[\bz_k^\top \bDelta^\top \bP^\top \bP \bDelta \bz_k\right] &= \tr\!\left(\bDelta^\top \bP^\top \bP \bDelta \, \E[\bz_k \bz_k^\top]\right) \\
        &\stackrel{(b)}{=} \tr\!\left(\bP \bDelta (\bG_c + \boldsymbol{\mu}_c\boldsymbol{\mu}_c^\top) \bDelta^\top \bP^\top\right),
        \label{eq:delta_trace}
        \end{aligned}
    \end{equation}
    where (a) follows from the identity $\E[\bz^\top\mathbf{M}\,\bz] = \tr(\mathbf{M}\,\E[\bz\,\bz^\top])$ and (b) uses the decomposition of the second moment $\E[\bz_k\bz_k^\top] = \bG_c + \boldsymbol{\mu}_c\boldsymbol{\mu}_c^\top$ into covariance and mean outer product.
    
    The dynamical discriminability $\delta^2$ captures two distinct sources of discriminative information. The first one is the operator difference $\bDelta = \Ac - \Ah$ which encodes how differently the two dynamical regimes evolve the state. If $\Ac = \Ah$, then $\bDelta = \mathbf{0}$ and $\delta^2 = 0$, meaning the two systems are indistinguishable from their dynamics regardless of the embedding. The second aspect is the state distribution, characterized jointly by the covariance $\bG_c$ and the mean $\boldsymbol{\mu}_c$, which determines which regions of the state space the trajectory typically visits. Even if $\bDelta$ is large in some directions, those directions only contribute to $\delta^2$ if the stationary distribution places significant probability mass along them. However, if $\bG_c$ concentrates the state along directions where $\bDelta$ acts strongly, or if the stationary mean $\boldsymbol{\mu}_c$ has a large component in a direction where $\bDelta$ acts, then $\delta^2$ will be large.
    
    A natural question arises: how reliably can a classifier distinguish whether $\mathbf{y}_k$ is drawn under $\mathcal{H}_c$ or $\mathcal{H}_h$? In the next section, we prove theoretical bounds on this DS discrimination task.
    
    %%%%%%%%%%%%%%%%%%%%%%%%%%%%%
    \section{Sub-optimality of disregarding dynamics}
    \label{sec:thm1}
    % ============================================================
    
    Here we establish that the marginal stationary distributions of the observables $\by_k$ under the two hypotheses can have bounded divergence even when the dynamics are very different. This provides a theoretical classification error floor when the dynamic behind the generation of $\by_k$ is disregarded.
    
    \begin{theorem}[marginal density analysis limitation]
    \label{thm:static_limit}
    Under Assumptions~\ref{ass:stability} and~\ref{ass:gaussian}, let $p_c = \mathcal{N}(\bP\boldsymbol{\mu}_c, \boldsymbol{\Gamma}_c)$ and $p_h = \mathcal{N}(\bP\boldsymbol{\mu}_h, \boldsymbol{\Gamma}_h)$ be the stationary distributions of the observables $\by_k$ under $\mathcal{H}_c$ and $\mathcal{H}_h$, respectively, where $\boldsymbol{\Gamma}_c = \bP\bG_c\bP^\top$ and $\boldsymbol{\Gamma}_h = \bP\bG_h\bP^\top$ are the observable-space stationary covariances. Defining $\norm{\cdot}$ as the Mahalanobis distance, the total variation (TV) distance $d_{\TV}(p_c, p_h)$ between $p_c$ and $p_h$ satisfies:
    \begin{equation}
        d_{\TV}(p_c, p_h) \leq \frac{1}{2} \norm{\boldsymbol{\Gamma}_c^{-1/2}(\boldsymbol{\Gamma}_c - \boldsymbol{\Gamma}_h)\boldsymbol{\Gamma}_c^{-1/2}}_F + \frac{1}{2}\mnorm{\bP(\boldsymbol{\mu}_c - \boldsymbol{\mu}_h)}{\boldsymbol{\Gamma}_h^{-1}},
        \label{eq:tv_bound}
    \end{equation}
    and the Bayes error of any classifier using only the marginal distribution satisfies:
    \begin{equation}
        P_{\mathrm{error}} \geq \frac{1}{2}\left(1 - d_{\TV}(p_c, p_h)\right).
    \end{equation}
    Moreover, there exist pairs $(\Ac, \Ah)$ and noise parameters $(\mathbf{m}_c, \bSigma_c, \mathbf{m}_h, \bSigma_h)$ with $\norm{\bDelta}_F$ bounded away from zero such that both $\norm{\boldsymbol{\Gamma}_c - \boldsymbol{\Gamma}_h}_F$ and $\norm{\bP(\boldsymbol{\mu}_c - \boldsymbol{\mu}_h)}$ are arbitrarily small.
    \end{theorem}
    
    \begin{proof}
    See Appendix \ref{app:A}.  
    \end{proof}

    This result justifies the situation encountered in practice where the SVD mode distributions (marginals) of factual and hallucinated embeddings are nearly indistinguishable, but the transition dynamics differ. This was observed empirically in \citep{wilson2026low} where the first two SVD modes of the token embeddings showed distributional differences that are insufficient for hallucination detection. Note that when the noise means differ ($\mathbf{m}_c \neq \mathbf{m}_h$), the second term in~(\ref{eq:tv_bound}) can be nonzero even for scaled rotations with equal noise covariance, providing limited discriminability through the mean shift alone.
    
    Rather than classifying based on marginal distributions alone, we ask whether incorporating the prediction error as described in Section \ref{sec:DS-for-classification} leads to a significant improvement in classification performance, which is the focus of the next section.
    
    \section{Exponential Separability via Dynamical System Prediction}
    \label{sec:thm2}
    
    In this section, we prove that the dynamical prediction approach achieves exponentially decaying classification error.
    
    \begin{theorem}[Dynamical separability]
    \label{thm:dynamical_sep}
    Under Assumptions~\ref{ass:stability} and~\ref{ass:gaussian}, suppose $\delta^2 + \beta_c > 0$ where $\delta^2$ and $\beta_c$ are defined in~(\ref{eq:discriminability}) and~(\ref{eq:beta_c}). Let $\eta \in \R$ be an arbitrary classification threshold and define the mean discriminability gap under $\mathcal{H}_c$ as
    \begin{equation}
        \mu_L \equiv \E[\Delta\mathcal{E} \mid \mathcal{H}_c] > 0.
    \end{equation}
    Then, for any threshold $\eta < \mu_L$, the misclassification probability satisfies:
    \begin{equation}
        \Prob\!\left(\Delta\mathcal{E} < \eta \mid \mathcal{H}_c\right) \leq 2\exp\!\left(-\frac{(L-1)(\mu_L - \eta)^2}{C_1\left(\sigma_{\max,c}^2(\delta^2 + |\beta_c|) + \sigma_{\max,c}^2 R^2 \lambda_{\max}(\bG_c) + P_m R^2\lambda_{\max}(\bG_c)\right)}\right),
        \label{eq:thm2_general}
    \end{equation}
    where $\sigma_{\max,c}^2 = \lambda_{\max}(\bP\bSigma_c\bP^\top)$, $R = \norm{\bP\bDelta}$, $P_m = \norm{\bP\mathbf{m}_c}^2$ and the constant $C_1 > 0$. An similar bound holds under $\mathcal{H}_h$ with $\boldsymbol{\mu}_c, \bG_c, \bSigma_c, \mathbf{m}_c$ replaced by $\boldsymbol{\mu}_h, \bG_h, \bSigma_h, \mathbf{m}_h$, and with the event $\{\Delta\mathcal{E} \geq \eta\}$.
    \end{theorem}
    
    \begin{proof}
    See Appendix \ref{app:B}.  
    \end{proof}
    
    Result~\ref{thm:dynamical_sep} establishes that the misclassification probability decays exponentially in the sequence length $L$, but the bound involves several variables (i.e., $\delta^2$, $\beta_c$, $\sigma_{\max,c}^2$, $R$, $\lambda_{\max}(\bG_c)$) whose interplay may obscure the essential scaling behavior. To present the result into a single interpretable quantity, we introduce the \emph{separability factor} $S(L)$, defined as the standardized mean of the test statistic $\Delta\mathcal{E}$ under $\mathcal{H}_c$. This quantity, borrowed from signal detection theory, measures how many standard deviations the expected score is separated from zero: when $S(L)$ is large, the distribution of $\Delta\mathcal{E}$ under $\mathcal{H}_c$ is concentrated far from the decision boundary, making misclassification unlikely. The following corollary shows that $S(L)$ grows as $\sqrt{L}$, providing a direct link between sequence length and classification reliability, and yields an explicit formula for the minimum number of tokens needed to achieve a target error rate.

    \begin{corollary}[Convergence rate and minimum sequence length]
    \label{cor:rate}
    Under the conditions of Result~\ref{thm:dynamical_sep}, the separability factor $S(L)$ of $\Delta\mathcal{E}$ under $\mathcal{H}_c$ grows as $\sqrt{L}$:
    \begin{equation}
        S(L) \equiv \frac{\E[\Delta\mathcal{E} \mid \mathcal{H}_c]}{\sqrt{\mathrm{Var}(\Delta\mathcal{E} \mid \mathcal{H}_c)}} = \Theta\!\left(\sqrt{L} \cdot \frac{(\delta^2 + \beta_c)}{\sigma_{\max,c}^2}\right).
        \label{eq:detectability}
    \end{equation}
    Consequently, for any threshold $\eta$ satisfying $|\eta| \leq \mu_L/2$, the number of tokens needed for classification with error probability $\leq \epsilon$ scales as:
    \begin{equation}
        L^* = \mathcal{O}\!\left(\frac{\sigma_{\max,c}^4}{(\delta^2 + \beta_c)^2}\log\frac{1}{\epsilon}\right).
        \label{eq:sample_complexity}
    \end{equation}
    \end{corollary}
    
    \begin{proof}
    See Appendix \ref{App:C}.  
    \end{proof}
    \vspace{-0.2cm}
    This result explains the empirical observation that larger embedding models (e.g., Llama-3) achieve near-perfect classification even with short sequences, while smaller models (e.g., Jina-v5 with 30M parameters) require longer sequences.
    
    \begin{figure}[h!]
        %\centering  
        \hspace{-1.2cm}
        \includegraphics[scale=0.52]{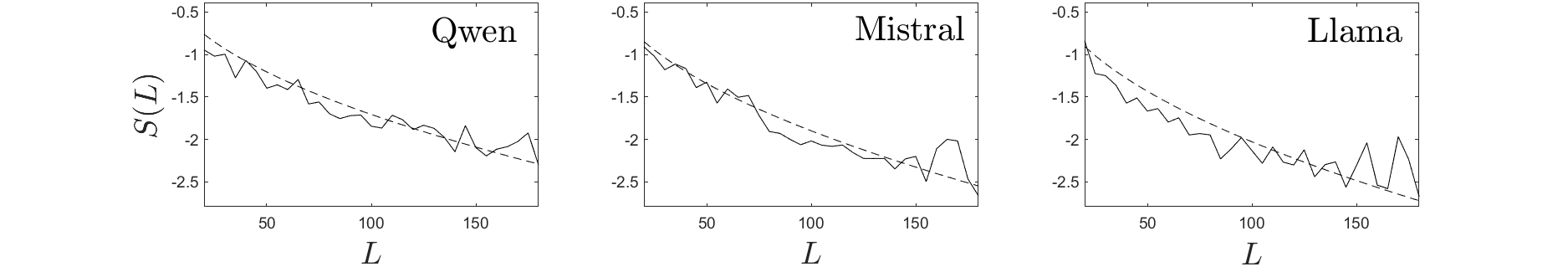}    
        \caption{Estimates of the separability factor (solid line) from \eqref{eq:detectability} for three different text embedding models for the HaluEval~\cite{li2023halueval} dataset. Dashed lines of the form $K \sqrt{L}$ are fit to the data, where $K$ is a constant used for fitting. }
        \label{fig:scaling}
        %\vspace{-0.1cm}
    \end{figure}
    
    Extended DMD is applied to the HaluEval~\cite{li2023halueval} dataset to obtain predictive models of the form \eqref{eq:Hc} and \eqref{eq:Hh} using token embeddings from \texttt{Qwen}, \texttt{Mistral} and \texttt{Llama-3}.  On the remaining 8000 hallucinated responses not used for model inference, the differential residual score $\Delta \mathcal{E}$ from \eqref{eq:resid} is computed for 8000 hallucinated responses.  $S(L)$ is computed for responses of different length, $L$, and plotted in Fig. \ref{fig:scaling}.  As predicted by Corollary \ref{cor:rate}, the separability factor is proportional to $\sqrt{L}$.
    
    \section{Cross-Embedding Generalization}
    \label{sec:cross}
    
    An important practical question is whether the dynamical models fitted using one embedding model can transfer to another. It has been show in (\cite{wilson2026low}) that fitting Koopman operators $\Ac^{(\boldsymbol{\theta}_1)}$ and $\Ah^{(\boldsymbol{\theta}_1)}$ on token embeddings from one model parametrized by $\boldsymbol{\theta}_1$ then run prediction on token embeddings from a different model parametrized by $\boldsymbol{\theta}_2$ achieve above-chance accuracy despite operating in an entirely different observation 
    space. We now provide a theoretical explanation for this phenomenon.
    
    Consider two embedding models $G_2^{(\boldsymbol{\theta}_1)}: \mathcal{Q} \to \R^{M_1}$ and $G_2^{(\boldsymbol{\theta}_2)}: \mathcal{Q} \to \R^{M_2}$, where $\mathcal{Q}$ 
    denotes the token space and $\boldsymbol{\theta}_1, \boldsymbol{\theta}_2$ parametrize the respective model architectures and weights. Each embedding model 
    induces its own observation space and, consequently, its own Koopman operator approximations. We write:
    \begin{align}
        \by_k^{(\boldsymbol{\theta}_i)} &= G_2^{(\boldsymbol{\theta}_i)}(q_k) \in \R^{M_i}, \qquad i \in \{1, 2\},
    \end{align}
    for the token embeddings under each model, and denote the corresponding fitted Koopman approximations by 
    $\Ac^{(\boldsymbol{\theta}_i)}, \Ah^{(\boldsymbol{\theta}_i)}$ with operator differences $\bDelta^{(\boldsymbol{\theta}_i)} = \Ac^{(\boldsymbol{\theta}_i)} - \Ah^{(\boldsymbol{\theta}_i)}$.
    
    The central question is: under what conditions on the pair $(G_2^{(\boldsymbol{\theta}_1)}, G_2^{(\boldsymbol{\theta}_2)})$ can dynamical models fitted in the $\boldsymbol{\theta}_1$-observation space provide discriminative information about trajectories observed in the $\boldsymbol{\theta}_2$-observation space?
    
    \begin{assumption}[Approximate intertwining]
    \label{ass:intertwining}
    There exists a linear map $\mathbf{T} \in \R^{M_2 \times M_1}$ such that the two embedding models are approximately related by:
    \begin{equation}
        G_2^{(\boldsymbol{\theta}_2)}(q) = \mathbf{T}\, G_2^{(\boldsymbol{\theta}_1)}(q) + \boldsymbol{r}(q), \qquad \forall\, q \in \mathcal{Q},
        \label{eq:intertwining}
    \end{equation}
    where $\boldsymbol{r}: \mathcal{Q} \to \R^{M_2}$ is a residual satisfying $\E[\norm{\boldsymbol{r}(q_k)}^2] \leq \chi^2$ along the token 
    sequences.
    \end{assumption}
    
    The map $\mathbf{T}$ acts as a \emph{translation} between the two observation spaces: it captures the extent to which the geometric structure of one embedding is a linear transformation of the other. The residual $\boldsymbol{r}$ measures the \emph{intertwining defect}, i.e., the degree to which this linear relationship breaks down. When $\chi^2 = 0$, the two embeddings 
    are exactly linearly related, and any dynamical structure in one space is perfectly preserved in the other. When $\chi^2 > 0$, the transfer is approximate.
    
    \begin{definition}[Cross-embedding discriminability]
    \label{def:cross_discriminability}
    The \emph{cross-embedding discriminability} when fitting on $G_2^{(\boldsymbol{\theta}_1)}$ and testing on $G_2^{(\boldsymbol{\theta}_2)}$ is:
    \begin{equation}
        \delta^2_{\boldsymbol{\theta}_1 \to \boldsymbol{\theta}_2} \equiv \E\!\left[\norm{\bP^{(\boldsymbol{\theta}_2)} \bDelta^{(\boldsymbol{\theta}_1)} \bz_k^{(\boldsymbol{\theta}_2)}}^2\right],
        \label{eq:cross_discrim}
    \end{equation}
    where $\bz_k^{(\boldsymbol{\theta}_2)}$ is the lifted observable from the $\boldsymbol{\theta}_2$-embedding and $\bP^{(\boldsymbol{\theta}_2)}$ is the corresponding 
    projection onto the observable components.
    \end{definition}
    
    \begin{theorem}[Cross-embedding transfer bound]
    \label{prop:cross}
    Under Assumptions~\ref{ass:stability},~\ref{ass:gaussian}, and~\ref{ass:intertwining}, the cross-embedding discriminability 
    satisfies:
    \begin{equation}
        \delta^2_{\boldsymbol{\theta}_1 \to \boldsymbol{\theta}_2} \geq \sigma_{\min}^2(\mathbf{T})\,\delta^2(\boldsymbol{\theta}_1) - C_2\!\left(\chi^2 \norm{\bDelta^{(\boldsymbol{\theta}_1)}}^2 \lambda_{\max}\big(\bG_c^{(\boldsymbol{\theta}_2)}\big) + \chi^2 \delta^2(\bm{\theta}_1)\right),
        \label{eq:cross_bound}
    \end{equation}
    where $\sigma_{\min}(\mathbf{T})$ is the smallest singular value of $\mathbf{T}$, $C_2 > 0$ is a constant, and $\delta^2(\boldsymbol{\theta}_1) = \tr(\bP^{(\boldsymbol{\theta}_1)}\bDelta^{(\boldsymbol{\theta}_1)}\bG_c^{(\boldsymbol{\theta}_1)}(\bDelta^{(\boldsymbol{\theta}_1)})^\top(\bP^{(\boldsymbol{\theta}_1)})^\top) + \norm{\bP^{(\boldsymbol{\theta}_1)}\bDelta^{(\boldsymbol{\theta}_1)}\boldsymbol{\mu}_c^{(\boldsymbol{\theta}_1)}}^2$ 
    is the discriminability under $\boldsymbol{\theta}_1$.
    
    Consequently, the classification accuracy when fitting on $G_2^{(\boldsymbol{\theta}_1)}$ and testing on $G_2^{(\boldsymbol{\theta}_2)}$ satisfies:
    \begin{equation}
        \begin{aligned}
        \mathrm{Accuracy}_{\boldsymbol{\theta}_1 \to \boldsymbol{\theta}_2} &= \frac{1}{2}\Prob(\Delta\mathcal{E} \geq 0 \mid \mathcal{H}_c) + \frac{1}{2}\Prob(\Delta\mathcal{E} < 0 \mid \mathcal{H}_h)\\
        &\geq \frac{1}{2} + \frac{1}{2}\left(1 - 2\exp\!\left(-\frac{(L-1)\left(\delta^2_{\boldsymbol{\theta}_1 \to \boldsymbol{\theta}_2}\right)^2}{C_1\big(\sigma_{\max,c}^2\, \delta^2_{\boldsymbol{\theta}_1 \to \boldsymbol{\theta}_2} + \sigma_{\max,c}^2 \,R^2\, \lambda_{\max}(\bG_c^{(\boldsymbol{\theta}_2)})\big)}\right)\right).
        \label{eq:cross_accuracy}
        \end{aligned}
    \end{equation}
    \end{theorem}
    
    \begin{proof}
        See Appendix \ref{App:D}.
    \end{proof}

For empirical validation of the bound in~(\ref{eq:cross_bound}), we select a model pair (\texttt{Qwen3}, \texttt{F2LLM}) whose cross-transfer behavior in \citep{wilson2026low} was intriguing: fitting the DS on \texttt{Qwen3} embedding and testing on \texttt{F2LLM} embeddings led to a cross-model accuracy of $63.0\%$, while doing the opposite yielded an accuracy of $50.5\%$. We then evaluate the intertwining map $\mathbf{T}$ and defect $\chi^2$ between the two embedding spaces. As depicted in Fig.~\ref{fig:crossgen}, while both mapping directions exhibit well-conditioned singular value spectra for $\mathbf{T}$ across the subspace modes, their reconstruction residual errors differ drastically: the intertwining defect for mapping \texttt{F2LLM} $\to$ \texttt{Qwen3} ($\chi^2 \approx 8.8 \times 10^6$) is nearly double that of \texttt{Qwen3} $\to$ \texttt{F2LLM} ($\chi^2 \approx 4.8 \times 10^6$). This aligns directly with our theoretical bound in Result~\ref{prop:cross}: when transferring from a lower-density embedding space (e.g., \texttt{F2LLM}) to a denser and compact space (e.g., \texttt{Qwen3}), the large intertwining defect $\chi^2$ acts as a heavy penalty term that suppresses the transferable discriminability $\delta^2_{\boldsymbol{\theta}_1 \to \boldsymbol{\theta}_2}$ toward zero, causing accuracy to drop to $50.5\%$. However, Transferring from \texttt{Qwen3} to \texttt{F2LLM} incurs a substantially smaller defect, preserving sufficient dynamical separation to maintain an accuracy of $63.0\%$ without refitting.\vspace{-0.3cm}

\begin{figure}[h!]
        \centering  
        \includegraphics[scale=0.39]{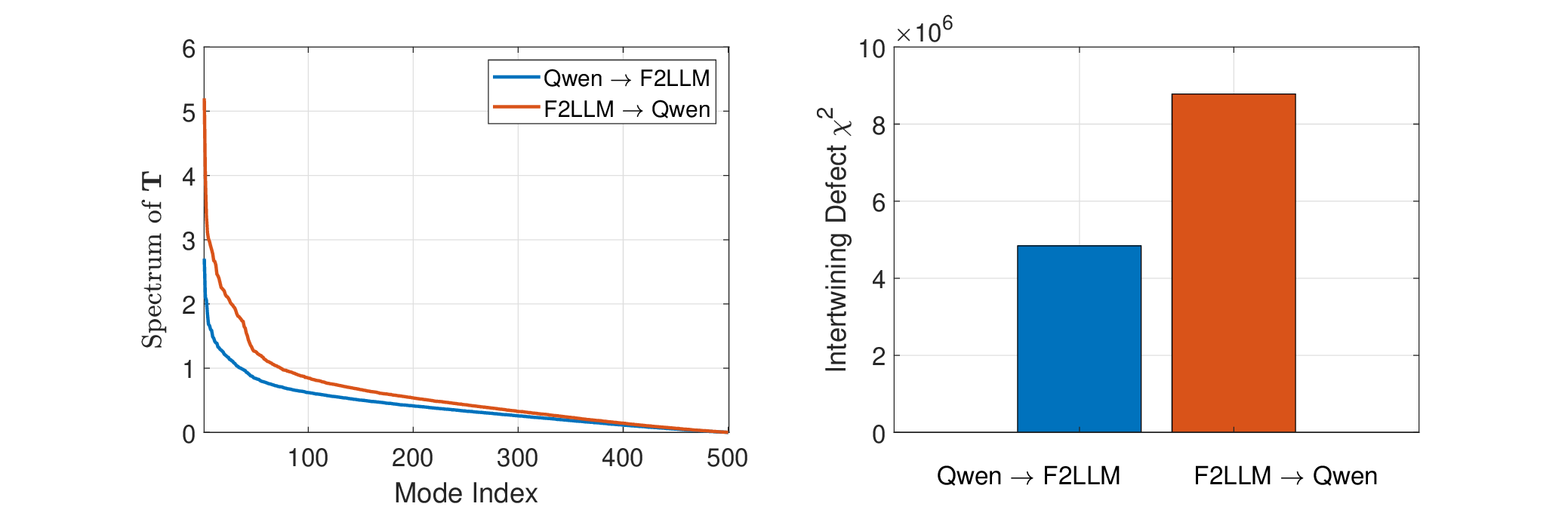}
        %\vspace{-0.2cm}
        \caption{Empirical validation of cross-embedding transfer between \texttt{Qwen3} and \texttt{F2LLM}. \textit{Left:} Singular values of the linear intertwining map $\mathbf{T}$ across SVD modes. \textit{Right:} The reconstruction residual error, i.e., intertwining defect $\chi^2$ between the two embedding models.}
        \label{fig:crossgen}
    \end{figure}
\vspace{-0.2cm}
\section{Conclusion}
We have developed a mathematical framework for understanding the distinguishability guarantees of two dynamical systems to classify LLM responses in the embedding space. First, we establish that classifiers relying solely on the marginal distribution of token embeddings face an irreducible accuracy floor since two dynamical systems can produce nearly identical stationary distributions while having completely different dynamical evolution. Second, we showed that exploiting temporal dynamics through prediction residuals overcomes this limitation, with classification error decaying exponentially in sequence length at a rate governed by the dynamical distinguishability of the two systems. This provides practical guidance for choosing the minimum response length needed for reliable detection. Third, we addressed the surprising phenomenon of cross-embedding transfer, demonstrating that dynamical models fitted on one embedding can transfer to another when the embedding spaces share similar geometric structure. The transfer quality depends on affine geometric compatibility, the fitted discriminability space, and the extent of distortion between the fitted and target embedding spaces. This explains why some embedding pairs transfer well while others do not. Together, these results establish guarantees on the distinguishability between dynamical systems theory for LLM token generations, shifting the focus from static representations of individual tokens to the dynamical patterns that emerge across the embedding trajectory of a LLM response.

\textbf{Limitations and future directions}: Our linear Koopman approximation could be extended to nonlinear models (e.g., kernel EDMD or neural network-based Koopman methods) to capture richer dynamical structure. The Gaussian noise assumption which is realistic after testing it on real-world datasets could be relaxed to heavy tail distributions to broaden the applicability of the obtained bounds. Future work can also devise a joint optimization over the embedding projection and the Koopman operators to seeking the projection that maximizes $\delta^2$ for systematic embedding selection. Such an algorithm will connect to the classical distinction between variance-maximizing and discriminability-maximizing projections.

\clearpage
%\subsubsection*{Acknowledgments}
%Use unnumbered third level headings for the acknowledgments. All acknowledgments, including those to funding agencies, go at the end of the paper.

\bibliographystyle{iclr2026_conference}
\bibliography{refs}

\clearpage
%\appendix
\section*{Appendices}
\setcounter{section}{0}
\renewcommand{\thesection}{\Alph{section}}
\refstepcounter{section}
\subsection*{Appendix A: Proof of Static analysis limitation}\label{app:A}

For two multivariate Gaussians $p_c = \mathcal{N}(\bP\boldsymbol{\mu}_c, \boldsymbol{\Gamma}_c)$ and $p_h = \mathcal{N}(\bP\boldsymbol{\mu}_h, \boldsymbol{\Gamma}_h)$, the Kullback-Leibler divergence is:
\begin{equation}
    {\KL}(p_c \| p_h) = \frac{1}{2}\left[\tr(\boldsymbol{\Gamma}_h^{-1}\boldsymbol{\Gamma}_c) - d + \log\frac{\det\boldsymbol{\Gamma}_h}{\det\boldsymbol{\Gamma}_c} + \mnorm{\bP(\boldsymbol{\mu}_c - \boldsymbol{\mu}_h)}{\boldsymbol{\Gamma}_h^{-1}}^2\right].
\end{equation}
By Pinsker's inequality, $d_{\TV}(p_c, p_h) \leq \sqrt{\frac{1}{2}{\KL}(p_c \| p_h)}$. We bound the covariance and mean contributions separately. Writing $\boldsymbol{\Gamma}_h = \boldsymbol{\Gamma}_c + \mathbf{E}$ where $\mathbf{E} = \boldsymbol{\Gamma}_h - \boldsymbol{\Gamma}_c$, and using the inequality $\log\det(\bI + \mathbf{M}) \leq \tr(\mathbf{M})$ for $\norm{\mathbf{M}} < 1$, the covariance-only part of the KL divergence satisfies:
\begin{equation}
    {\KL}^{\hspace{-0.35cm}\mathrm{cov}} \equiv \frac{1}{2}\left[\tr(\boldsymbol{\Gamma}_h^{-1}\boldsymbol{\Gamma}_c) - d + \log\frac{\det\boldsymbol{\Gamma}_h}{\det\boldsymbol{\Gamma}_c}\right] \leq \frac{1}{4}\norm{\boldsymbol{\Gamma}_c^{-1/2}\mathbf{E}\boldsymbol{\Gamma}_c^{-1/2}}_F^2.
\end{equation}
The mean contribution is the squared Mahalanobis distance:
\begin{equation}
    {\KL}^{\hspace{-0.35cm}\mathrm{mean}} \equiv \frac{1}{2}\mnorm{\bP(\boldsymbol{\mu}_c - \boldsymbol{\mu}_h)}{\boldsymbol{\Gamma}_h^{-1}}^2.
\end{equation}
Using $d_{\TV} \leq \sqrt{\frac{1}{2}(D_{\KL}^{\mathrm{cov}} + {\KL}^{\hspace{-0.35cm}\mathrm{mean}})} \leq \sqrt{\frac{1}{2}{\KL}^{\hspace{-0.35cm}\mathrm{cov}}} + \sqrt{\frac{1}{2}{\KL}^{\hspace{-0.35cm}\mathrm{mean}}}$, we obtain the bound in~(\ref{eq:tv_bound}).

\noindent Recall that the stationary covariances satisfy:
\begin{align}
    \boldsymbol{\Gamma}_c &= \bP\bG_c\bP^\top = \bP\left(\sum_{j=0}^{\infty} \Ac^j \bSigma_c (\Ac^j)^\top\right)\bP^\top, \\
    \boldsymbol{\Gamma}_h &= \bP\bG_h\bP^\top = \bP\left(\sum_{j=0}^{\infty} \Ah^j \bSigma_h (\Ah^j)^\top\right)\bP^\top,
\end{align}
according to which the difference $\mathbf{E} = \boldsymbol{\Gamma}_h - \boldsymbol{\Gamma}_c$ satisfies:
\begin{equation}
    \norm{\mathbf{E}}_F \leq \sum_{j=0}^{\infty} \norm{\bP(\Ah^j \bSigma_h (\Ah^j)^\top - \Ac^j \bSigma_c (\Ac^j)^\top)\bP^\top}_F.
\end{equation}
The summand decomposes as $\Ah^j\bSigma_h(\Ah^j)^\top - \Ac^j\bSigma_c(\Ac^j)^\top = (\Ah^j - \Ac^j)\bSigma_h(\Ah^j)^\top + \Ac^j(\bSigma_h - \bSigma_c)(\Ah^j)^\top + \Ac^j\bSigma_c((\Ah^j)^\top - (\Ac^j)^\top)$. Using the identity $\Ah^j - \Ac^j = \sum_{l=0}^{j-1} \Ah^{j-1-l}\bDelta \Ac^l$ and the submultiplicativity of norms:
\begin{equation}
    \norm{\Ah^j - \Ac^j} \leq j \cdot \rho_{\max}^{j-1} \norm{\bDelta},
\end{equation}
where $\rho_{\max} = \max(\rho(\Ac), \rho(\Ah))$. After bounding the telescoping sum, we obtain:
\begin{equation}
    \norm{\mathbf{E}}_F \leq \frac{C \norm{\bDelta} (\norm{\bSigma_c} + \norm{\bSigma_h})}{(1-\rho_{\max})^3} + \frac{\norm{\bSigma_c - \bSigma_h}_F}{1-\rho_{\max}^2},
\end{equation}
where $C > 0$ is a constant. The first term captures the operator difference contribution and the second captures the noise covariance difference. For the mean difference, $\bP(\boldsymbol{\mu}_c - \boldsymbol{\mu}_h) = \bP[(\bI - \Ac)^{-1}\mathbf{m}_c - (\bI - \Ah)^{-1}\mathbf{m}_h]$, which can be bounded using the resolvent identity $(\bI - \Ac)^{-1} - (\bI - \Ah)^{-1} = (\bI - \Ac)^{-1}\bDelta(\bI - \Ah)^{-1}$.

\renewcommand{\thesection}{\Alph{section}}
\refstepcounter{section}
\subsection*{Appendix B: Proof of dynamical separability Result}\label{app:B}

Under $\mathcal{H}_c$, the true dynamics are $\bz_{k+1} = \Ac\bz_k + \bxi_k^{(c)}$ with $\bxi_k^{(c)} \sim \mathcal{N}(\mathbf{m}_c, \bSigma_c)$. Write $\bxi_k^{(c)} = \mathbf{m}_c + \tilde{\bxi}_k$ where $\tilde{\bxi}_k \sim \mathcal{N}(\mathbf{0}, \bSigma_c)$, and decompose $\bz_k = \boldsymbol{\mu}_c + \tilde{\bz}_k$ where $\tilde{\bz}_k$ is the zero-mean centered process satisfying $\tilde{\bz}_{k+1} = \Ac\tilde{\bz}_k + \tilde{\bxi}_k$ with $\E[\tilde{\bz}_k] = \mathbf{0}$ and $\mathrm{Cov}(\tilde{\bz}_k) = \bG_c$.

The prediction error using the correct model is $\epsilon_{c,k}^2 = \norm{\bP\bxi_k^{(c)}}^2$, while the prediction error using the wrong model is:
\begin{equation}
    \epsilon_{h,k}^2 = \norm{\bP\bDelta\bz_k + \bP\bxi_k^{(c)}}^2 = \norm{\bP\bxi_k^{(c)}}^2 + 2(\bxi_k^{(c)})^\top\bP^\top\bP\bDelta\bz_k + \norm{\bP\bDelta\bz_k}^2.
\end{equation}
The difference of cumulative squared errors therefore decomposes as:
\begin{equation}
    S_h - S_c = \underbrace{\sum_{k=1}^{L-1}\norm{\bP\bDelta\bz_k}^2}_{\triangleq\,T_1} + \underbrace{2\sum_{k=1}^{L-1}(\bxi_k^{(c)})^\top\bP^\top\bP\bDelta\bz_k}_{\triangleq\,T_2}.
    \label{eq:ShSc_decomp}
\end{equation}
The signal term $T_1$ has expectation $\E[T_1] = (L-1)\delta^2$ in the stationary regime. Since $\bxi_k^{(c)}$ is independent of $\bz_k$, the cross term $T_2$ has expectation:
\begin{equation}
    \E[T_2] = 2(L-1)\mathbf{m}_c^\top\bP^\top\bP\bDelta\boldsymbol{\mu}_c = (L-1)\beta_c.
\end{equation}
The total expected gap is therefore $\E[S_h - S_c] = (L-1)(\delta^2 + \beta_c)$.

To establish concentration, we decompose $T_2$ into its mean and centered parts: $T_2 = (L-1)\beta_c + \widetilde{T}_2$ where
\begin{equation}
    \widetilde{T}_2 = 2\sum_{k=1}^{L-1}\left[(\bxi_k^{(c)})^\top\bP^\top\bP\bDelta\bz_k - \mathbf{m}_c^\top\bP^\top\bP\bDelta\boldsymbol{\mu}_c\right] = 2\sum_{k=1}^{L-1}\left[\tilde{\bxi}_k^\top\bP^\top\bP\bDelta\bz_k + \mathbf{m}_c^\top\bP^\top\bP\bDelta\tilde{\bz}_k\right].
\end{equation}
The centered residual $\widetilde{T}_2$ has zero mean. For the first sum in $\widetilde{T}_2$, conditionally on the trajectory $\{\bz_k\}_{k=1}^{L-1}$, the terms $\{\tilde{\bxi}_k^\top\bP^\top\bP\bDelta\bz_k\}$ are independent zero-mean Gaussians with variances $\norm{\bSigma_c^{1/2}\bP^\top\bP\bDelta\bz_k}^2 \leq \sigma_{\max,c}^2\norm{\bP\bDelta\bz_k}^2$. The conditional variance of this part is therefore bounded by $4\,\sigma_{\max,c}^2 \,T_1$. For the second sum, $\mathbf{m}_c^\top\bP^\top\bP\bDelta\tilde{\bz}_k$ is a linear functional of the geometrically mixing process $\tilde{\bz}_k$; its partial sums have variance of order $\mathcal{O}(L\norm{\bP\mathbf{m}_c}^2 R^2\lambda_{\max}(\bG_c))$.

Since $\norm{\bP\bDelta\bz_k}^2 = \bz_k^\top\bDelta^\top\bP^\top\bP\bDelta\bz_k$ is a quadratic form in the Gaussian vector $\bz_k$, and the sequence $\{\bz_k\}$ has geometrically decaying correlations under the stability assumption $\rho(\Ac) < 1$, the Hanson--Wright inequality (\cite{rudelson2013hanson}) applied to blocks of approximately independent iterates yields:
\begin{equation}
    \Prob\!\left(\left|T_1 - (L-1)\delta^2\right| > t\right) \leq 2\exp\!\left(-c_1\min\!\left\{\frac{t^2}{(L-1)R^4\lambda_{\max}(\bG_c)^2},\; \frac{t}{R^2\lambda_{\max}(\bG_c)}\right\}\right),
    \label{eq:T1_hw}
\end{equation}
where $c_1 > 0$ is a constant. The Gaussian tail bound on $\widetilde{T}_2$, combined with the mixing bound on the linear part, gives:
\begin{equation}
    \Prob\!\left(|\widetilde{T}_2| > t \;\Big|\; \{\bz_k\}\right) \leq 2\exp\!\left(-\frac{t^2}{8\sigma_{\max,c}^2 T_1 + C'\norm{\bP\mathbf{m}_c}^2R^2\lambda_{\max}(\bG_c)(L-1)}\right),
    \label{eq:T2_gauss}
\end{equation}
for a constant $C' > 0$. Restricting to the high-probability event where $T_1$ is close to its mean, setting the deviation $t = (L-1)\alpha$ with $\alpha = (\delta^2 + \beta_c)/4$, and integrating over the trajectory yields concentration of $S_h - S_c$ around $(L-1)(\delta^2 + \beta_c)$.

To pass from $S_h - S_c$ to $\Delta\mathcal{E} = \sqrt{S_h} - \sqrt{S_c}$, we use the identity:
\begin{equation}
    \Delta\mathcal{E} = \frac{S_h - S_c}{\sqrt{S_h} + \sqrt{S_c}}.
    \label{eq:DE_identity}
\end{equation}
The denominator is controlled by the concentration of $S_c$ around its mean. Since $S_c = \sum_{k=1}^{L-1}\norm{\bP\bxi_k^{(c)}}^2$ is a sum of i.i.d.\ terms with $\E[S_c] = (L-1)[\tr(\bP\bSigma_c\bP^\top) + \norm{\bP\mathbf{m}_c}^2]$, standard sub-exponential concentration gives $S_c \leq 2(L-1)[d\sigma_{\max,c}^2 + \norm{\bP\mathbf{m}_c}^2]$ with high probability. On this event, together with $S_h \leq S_c + (S_h - S_c)$, the denominator in~\eqref{eq:DE_identity} satisfies:
\begin{equation}
    \sqrt{S_h} + \sqrt{S_c} \leq 2\sqrt{S_h} \leq 2\sqrt{2(L-1)(d\sigma_{\max,c}^2 + \norm{\bP\mathbf{m}_c}^2) + (L-1)(\delta^2 + |\beta_c|)},
\end{equation}
so that:
\begin{equation}
    \mu_L \equiv \E[\Delta\mathcal{E} \mid \mathcal{H}_c] \geq \frac{\sqrt{L-1}\,(\delta^2 + \beta_c)}{8\,\sqrt{2(d\sigma_{\max,c}^2 + \norm{\bP\mathbf{m}_c}^2) + \delta^2 + |\beta_c|}} > 0.
    \label{eq:mu_lower}
\end{equation}
The same concentration applied to the fluctuations of $\Delta\mathcal{E}$ around $\mu_L$ shows that for any $\eta < \mu_L$:
\begin{equation}
    \begin{aligned}
    &\Prob\!\left(\Delta\mathcal{E} < \eta \mid \mathcal{H}_c\right)\leq \Prob\!\left(\Delta\mathcal{E} - \mu_L < \eta - \mu_L\right) \\
    &\leq 2\exp\!\left(-\frac{(L-1)(\mu_L - \eta)^2}{C_1(\sigma_{\max,c}^2(\delta^2 + |\beta_c|) + \sigma_{\max,c}^2 R^2\lambda_{\max}(\bG_c) + \norm{\bP\mathbf{m}_c}^2R^2\lambda_{\max}(\bG_c))}\right),
    \label{eq:general_eta_bound}
    \end{aligned}
\end{equation}
where $C_1 > 0$ is a constant absorbing the numerical factors from the preceding bounds. This establishes~(\ref{eq:thm2_general}).

\renewcommand{\thesection}{\Alph{section}}
\refstepcounter{section}
\subsection*{Appendix C: Proof of Corollary~\ref{cor:rate}}\label{App:C}

We compute the mean and variance of $\Delta\mathcal{E}$ by working through the auxiliary quantity $S_h - S_c$ and transferring the results via the identity $\Delta\mathcal{E} = (S_h - S_c)/(\sqrt{S_h} + \sqrt{S_c})$.

From the decomposition~(\ref{eq:ShSc_decomp}), the mean of $S_h - S_c$ under $\mathcal{H}_c$ is $\E[S_h - S_c] = \E[T_1] + \E[T_2] = (L-1)(\delta^2 + \beta_c)$. For the variance, the centered residual $\widetilde{T}_2 = T_2 - (L-1)\beta_c$ gives:
\begin{equation}
    \mathrm{Var}(S_h - S_c) = \mathrm{Var}(T_1) + \mathrm{Var}(\widetilde{T}_2) + 2\,\mathrm{Cov}(T_1, \widetilde{T}_2).
\end{equation}
The signal term $T_1 = \sum_{k=1}^{L-1}\norm{\bP\bDelta\bz_k}^2$ is a sum of quadratic forms in the geometrically mixing Gaussian process $\{\bz_k\}$. Each summand has variance $\mathcal{O}(R^4\lambda_{\max}(\bG_c)^2 + R^2\norm{\bP\bDelta\boldsymbol{\mu}_c}^2\lambda_{\max}(\bG_c))$, and the geometric mixing ensures that the effective number of independent blocks is $\Theta(L)$, so $\mathrm{Var}(T_1) = \mathcal{O}\big(L \cdot (R^4\lambda_{\max}(\bG_c)^2 + R^2\norm{\bP\bDelta\boldsymbol{\mu}_c}^2\lambda_{\max}(\bG_c))\big)$. For the second cross-term, using the law of total variance, we have:
\begin{equation}
    \begin{aligned}
    \mathrm{Var}(\widetilde{T}_2) &= \E[\mathrm{Var}(\widetilde{T}_2 \mid \{\bz_k\})] + \mathrm{Var}(\E[\widetilde{T}_2 \mid \{\bz_k\}]) \\
    &\leq 4\,\sigma_{\max,c}^2(L-1)\delta^2 + C'(L-1)\norm{\bP\mathbf{m}_c}^2R^2\lambda_{\max}(\bG_c),
    \end{aligned}
\end{equation}

where the first term bounds the conditional variance from the zero-mean noise $\tilde{\bxi}_k$, and the second bounds the variance of the linear functional of $\tilde{\bz}_k$, and $C'>0$ is a constant. In the regime where the noise-driven term dominates, we have:
\begin{equation}
    \mathrm{Var}(S_h - S_c) = \Theta\!\left(L \cdot [\sigma_{\max,c}^2 \delta^2 + \norm{\bP\mathbf{m}_c}^2R^2\lambda_{\max}(\bG_c)]\right).
    \label{eq:var_ShSc}
\end{equation}

To transfer these to $\Delta\mathcal{E}$, note that $S_c = \sum_{k=1}^{L-1}\norm{\bP\bxi_k^{(c)}}^2$ concentrates around $(L-1)[\tr(\bP\bSigma_c\bP^\top) + \norm{\bP\mathbf{m}_c}^2] = \Theta(L[\sigma_{\max,c}^2 d + \norm{\bP\mathbf{m}_c}^2])$ by the law of large numbers. The denominator in the identity $\Delta\mathcal{E} = (S_h - S_c)/(\sqrt{S_h} + \sqrt{S_c})$ therefore satisfies:
\begin{equation}
    \sqrt{S_h} + \sqrt{S_c} = \Theta\!\left(\sqrt{L}\,\sqrt{\sigma_{\max,c}^2 d + \norm{\bP\mathbf{m}_c}^2}\right),
\end{equation}
in the typical regime where the model-mismatch signal is small relative to the total noise energy. Applying the delta method to the ratio, the mean and standard deviation of $\Delta\mathcal{E}$ are:
\begin{align}
    \E[\Delta\mathcal{E} \mid \mathcal{H}_c] &= \Theta\!\left(\sqrt{L}\cdot\frac{\delta^2 + \beta_c}{\sqrt{\sigma_{\max,c}^2 d + \norm{\bP\mathbf{m}_c}^2}}\right), \label{eq:mean_DE} \\
    \sqrt{\mathrm{Var}(\Delta\mathcal{E} \mid \mathcal{H}_c)} &= \Theta\!\left(\frac{\sqrt{\sigma_{\max,c}^2\delta^2 + \norm{\bP\mathbf{m}_c}^2R^2\lambda_{\max}(\bG_c)}}{\sqrt{\sigma_{\max,c}^2 d + \norm{\bP\mathbf{m}_c}^2}}\right). \label{eq:std_DE}
\end{align}
Taking the ratio and absorbing the dimension-dependent constants:
\begin{equation}
    S(L) = \frac{\E[\Delta\mathcal{E} \mid \mathcal{H}_c]}{\sqrt{\mathrm{Var}(\Delta\mathcal{E} \mid \mathcal{H}_c)}} = \Theta\!\left(\sqrt{L}\cdot\frac{\delta^2 + \beta_c}{\sigma_{\max,c}^2}\right),
    \label{eq:dprime_derived}
\end{equation}
establishing~(\ref{eq:detectability}). For the sample complexity, the misclassification probability from Result~\ref{thm:dynamical_sep} decays as $\exp(-c\,S(L)^2) = \exp(-c\,L\,(\delta^2+\beta_c)^2/\sigma_{\max,c}^4)$. Setting this equal to $\epsilon$ and solving for $L$ gives:
\begin{equation}
    L^* = \frac{\sigma_{\max,c}^4}{c\,(\delta^2+\beta_c)^2}\log\frac{1}{\epsilon} = \mathcal{O}\!\left(\frac{\sigma_{\max,c}^4}{(\delta^2+\beta_c)^2}\log\frac{1}{\epsilon}\right),
\end{equation}
which corresponds to the desired result in~(\ref{eq:sample_complexity}).

\renewcommand{\thesection}{\Alph{section}}
\refstepcounter{section}
\subsection*{Appendix D: Proof of Cross-embedding transfer bound}\label{App:D}

Under the intertwining assumption~(\ref{eq:intertwining}), the $\boldsymbol{\theta}_2$-embedding of token $q_k$ can be written as 
$\by_k^{(\boldsymbol{\theta}_2)} = \mathbf{T}\by_k^{(\boldsymbol{\theta}_1)} + \boldsymbol{r}_k$ where $\boldsymbol{r}_k = \boldsymbol{r}(q_k)$. 
The prediction error when applying the $\boldsymbol{\theta}_1$-fitted operator $\Ah^{(\theta_1)}$ to the $\boldsymbol{\theta}_2$-trajectory is:
\begin{equation}
    \by_{k+1}^{(\boldsymbol{\theta}_2)} - \bP^{(\boldsymbol{\theta}_2)}\Ah^{(\boldsymbol{\theta}_1)}\bz_k^{(\boldsymbol{\theta}_2)} = \mathbf{T}\by_{k+1}^{(\boldsymbol{\theta}_1)} + \boldsymbol{r}_{k+1} - \bP^{(\boldsymbol{\theta}_2)}\Ah^{(\boldsymbol{\theta}_1)}\bz_k^{(\boldsymbol{\theta}_2)}.
\end{equation}
Under $\mathcal{H}_c$, the true dynamics in the
$\boldsymbol{\theta}_1$-space satisfy
$\by_{k+1}^{(\boldsymbol{\theta}_1)}
  = \bP^{(\boldsymbol{\theta}_1)}\Ac^{(\boldsymbol{\theta}_1)}
    \bz_k^{(\boldsymbol{\theta}_1)}
  + \bP^{(\boldsymbol{\theta}_1)}\bxi_k^{(\boldsymbol{\theta}_1)}$.
Substituting and rearranging, let
$\boldsymbol{\rho}_k
  \coloneqq \boldsymbol{\rho}_k\!\bigl(
    \boldsymbol{r}_k,\,
    \bP^{(\boldsymbol{\theta}_1)},\,
    \Ac^{(\boldsymbol{\theta}_1)}
  \bigr) \in \mathbb{R}^{n}$
denote the residual that depends on the observation noise
$\boldsymbol{r}_k$.
The model-mismatch component of the prediction error takes the form
\begin{equation}\label{eq:pred_error_mismatch}
  \boldsymbol{e}_k
  \;=\;
  \mathbf{T}\,\bP^{(\boldsymbol{\theta}_1)}
    \bDelta^{(\boldsymbol{\theta}_1)}
    \bz_k^{(\boldsymbol{\theta}_1)}
  \;+\;
  \boldsymbol{\rho}_k.
\end{equation}
Taking the expected squared norm and applying the reverse
triangle inequality
$\norm{\boldsymbol{a}+\boldsymbol{b}}^2
  \geq \norm{\boldsymbol{a}}^2
       - 2\norm{\boldsymbol{a}}\norm{\boldsymbol{b}}$
yields
\begin{align}
  \delta^2_{\boldsymbol{\theta}_1 \to \boldsymbol{\theta}_2}
  &\;\geq\;
    \E\!\left[
      \norm{
        \mathbf{T}\,\bP^{(\boldsymbol{\theta}_1)}
        \bDelta^{(\boldsymbol{\theta}_1)}
        \bz_k^{(\boldsymbol{\theta}_1)}
      }^2
    \right]
    \;-\;
    2\,\E\!\left[
      \norm{
        \mathbf{T}\,\bP^{(\boldsymbol{\theta}_1)}
        \bDelta^{(\boldsymbol{\theta}_1)}
        \bz_k^{(\boldsymbol{\theta}_1)}
      }\,
      \norm{\boldsymbol{\rho}_k}
    \right]
  \nonumber\\[6pt]
  &\;\geq\;
    \sigma_{\min}^2(\mathbf{T})\;\,
    \E\!\left[
      \norm{
        \bP^{(\boldsymbol{\theta}_1)}
        \bDelta^{(\boldsymbol{\theta}_1)}
        \bz_k^{(\boldsymbol{\theta}_1)}
      }^2
    \right]
    \;-\;
    C_2\,\chi^2
    \!\left(
      \norm{\bDelta^{(\boldsymbol{\theta}_1)}}^2\,
      \lambda_{\max}\!\bigl(
        \bG_c^{(\boldsymbol{\theta}_2)}
      \bigr)
      \;+\;
      \delta^2(\boldsymbol{\theta}_1)
    \right)
  \nonumber\\[6pt]
  &\;=\;
    \sigma_{\min}^2(\mathbf{T})\;\,
    \delta^2(\boldsymbol{\theta}_1)
    \;-\;
    C_2\,\chi^2
    \!\left(
      \norm{\bDelta^{(\boldsymbol{\theta}_1)}}^2\,
      \lambda_{\max}\!\bigl(
        \bG_c^{(\boldsymbol{\theta}_2)}
      \bigr)
      \;+\;
      \delta^2(\boldsymbol{\theta}_1)
    \right),
  \label{eq:delta_lower_bound}
\end{align}
where the second line uses the submultiplicativity
$\norm{\mathbf{T}\,\boldsymbol{x}} \geq
  \sigma_{\min}(\mathbf{T})\,\norm{\boldsymbol{x}}$
together with the bound
$\E\!\bigl[\norm{\boldsymbol{\rho}_k}^2\bigr]
  \leq C_2\,\chi^2\!\bigl(
    \norm{\bDelta^{(\boldsymbol{\theta}_1)}}^2\,
    \lambda_{\max}\!\bigl(\bG_c^{(\boldsymbol{\theta}_2)}\bigr)
    + \delta^2(\boldsymbol{\theta}_1)
  \bigr)$,
and the last equality identifies
$\E\!\bigl[
  \norm{
    \bP^{(\boldsymbol{\theta}_1)}
    \bDelta^{(\boldsymbol{\theta}_1)}
    \bz_k^{(\boldsymbol{\theta}_1)}
  }^2
\bigr] = \delta^2(\boldsymbol{\theta}_1)$. This establishes~(\ref{eq:cross_bound}).

The accuracy bound~(\ref{eq:cross_accuracy}) follows by substituting $\delta^2_{\boldsymbol{\theta}_1 \to \boldsymbol{\theta}_2}$ in place of $\delta^2$ in 
Result~\ref{thm:dynamical_sep} with $\eta = 0$. The cross-embedding setting introduces no structural change to the hypothesis testing framework: the test statistic $\Delta\mathcal{E}$ is still computed from prediction residuals, and the exponential 
concentration still holds provided the effective discriminability $\delta^2_{\boldsymbol{\theta}_1 \to \boldsymbol{\theta}_2}$ is positive. The total classification error under equal priors is:
\begin{equation}
    \begin{aligned}
        \Prob(\text{error}) &= \frac{1}{2}\Prob(\Delta\mathcal{E} < 0 \mid \mathcal{H}_c) + \frac{1}{2}\Prob(\Delta\mathcal{E} \geq 0 \mid \mathcal{H}_h) \\
        &\leq 2\exp\!\left(-\frac{(L-1)\left(\delta^2_{\boldsymbol{\theta}_1 \to \boldsymbol{\theta}_2}\right)^2}{C_1\!\left(\sigma_{\max,c}^2 \delta^2_{\boldsymbol{\theta}_1 \to \boldsymbol{\theta}_2} + \sigma_{\max,c}^2 R^2 \lambda_{\max}(\bG_c^{(\boldsymbol{\theta}_2)})\right)}\right),
    \end{aligned}
\end{equation}
where each term is bounded by Result~\ref{thm:dynamical_sep} applied in the $\boldsymbol{\theta}_2$-observation space with discriminability $\delta^2_{\boldsymbol{\theta}_1 \to \boldsymbol{\theta}_2}$. Since $\Prob(\text{error}) \leq 1/2$, the accuracy satisfies:
\begin{equation}
    \begin{aligned}
    \mathrm{Accuracy}_{\boldsymbol{\theta}_1 \to \boldsymbol{\theta}_2} &= 1 - \Prob(\text{error}) \\
    &\geq \frac{1}{2} + \frac{1}{2}\!\left(1 - 2\exp\!\left(-\frac{(L-1)\left(\delta^2_{\boldsymbol{\theta}_1 \to \boldsymbol{\theta}_2}\right)^2}{C_1\!\left(\sigma_{\max,c}^2 \delta^2_{\boldsymbol{\theta}_1 \to \boldsymbol{\theta}_2} + \sigma_{\max,c}^2 R^2 \lambda_{\max}(\bG_c^{(\boldsymbol{\theta}_2)})\right)}\right)\right),
    \end{aligned}
\end{equation}
which is precisely~(\ref{eq:cross_accuracy}).

\end{document}